%% file: draft.papier.tex
\documentclass{article}
\usepackage{spconf}

\usepackage{url}            
\usepackage{booktabs}       
\usepackage{amsfonts}       
\usepackage{nicefrac}       
\usepackage{microtype}      
\usepackage[dvipsnames]{xcolor}     

\usepackage{amsmath, amssymb,amsthm}
\usepackage{bbm}
\usepackage{graphicx, graphics, epsfig, color}
\usepackage{adjustbox}
\usepackage{subfigure}
\usepackage{tikz, pgfplots}
\usepackage{titlesec}
\usetikzlibrary{plotmarks}
\pgfplotsset{compat=1.18}
\usepackage{float}
\usepackage{multirow}
\usepackage{cuted, flushend}
\usepackage{dblfloatfix}
\usepackage{bm}
\usepackage{fancyhdr}
\usepackage{enumerate}
\usepackage{siunitx}
\usepackage{hyperref}
\usepackage{cleveref}
\usepackage{subfigure}
\usepackage{url}
\usepackage[skins,xparse,breakable]{tcolorbox}
\input{math_commands_I}

\renewcommand{\emph}[1]{\textit{#1}}
\title{Whitening Spherical Gaussian Mixtures in the Large-Dimensional Regime}
\name{M.~R.~M.~Boudjemaa,$^{1}$ \, A.~Kalle,$^{2}$ \, X.~Mai,$^{3}$ \, J.~H.~de M.~Goulart,$^{1}$ and C.~Févotte$^{1}$\thanks{This work is supported by the LabEx CIMI (ANR-11-LABX-0040) and the AI cluster ANITI (ANR-23-IACL-0002). Part of
this work was conducted while A.~Kalle was an MSc intern student at
IRIT.}}
\address{$^{1}$  IRIT, CNRS, Toulouse INP, Université de Toulouse \\
$^{2}$ List, CEA, Université Paris-Saclay \quad $^{3}$ IMT, Université de Toulouse - Jean Jaurès}
\begin{document}
\sloppy
\topmargin=0mm
\ninept
\maketitle
\begin{abstract}
Whitening is a classical technique in unsupervised learning that can facilitate estimation tasks by standardizing data.
An important application is the estimation of latent variable models via the decomposition of tensors built from high-order moments.
In particular, whitening orthogonalizes the means of a spherical Gaussian mixture model (GMM), thereby making the corresponding moment tensor orthogonally decomposable, hence easier to decompose. 
However, in the large-dimensional regime (LDR) where data are high-dimensional and scarce, the standard whitening matrix built from the sample covariance becomes ineffective because the latter is spectrally distorted.
Consequently, whitened means of a spherical GMM are no longer orthogonal.
Using random matrix theory, we derive exact limits for their dot products, which are generally nonzero in the LDR.
As our main contribution, we then construct a corrected whitening matrix that restores asymptotic orthogonality, allowing for performance gains in spherical GMM estimation.
\end{abstract}

\section{Introduction}
\label{sec:intro}

Whitening is a classical technique in unsupervised learning and latent variable model estimation \cite{bishop2006pattern}, \cite{comon1994independent}.
Concretely, it exploits the latent geometry of the data by constructing a linear transformation that eliminates pairwise correlations and normalizes variances.
This elementary yet powerful idea leads to a significant simplification of certain estimation tasks.
For instance, whitening reduces independent component analysis (ICA) to the estimation of an orthogonal rotation of the whitened signals \cite{comon1994independent}, \cite{hyvarinen2000independent}.
Recent estimators of latent variable models based on decomposing tensors built from high-order moments also rely on whitening, as an intermediate step that makes tensor decomposition solvable by spectral methods \cite{anandkumar2014tensor}, \cite{hsu2013learning}.

This strategy applies, in particular, to the estimation of spherical Gaussian mixture models (GMM) \cite{hsu2013learning}, \cite{moitra2010settling}, \cite{kalai2010efficiently}. For concreteness,  consider a spherical Gaussian mixture model (GMM) with $K$ components, weights $\{\omega_k\}_{k=1}^K$ (positive, summing to one), means $\{\bmu_k\}_{k=1}^K\subset\RR^P$, and a common variance $\sigma^2 > 0$, so that an observation $\x_n$ is distributed as
\begin{equation}
 \label{GMM}
    \x_n\sim\sum\nolimits_{k=1}^K \omega_k\,\mathcal{N}(\bmu_k,\sigma^2\mI_P).
\end{equation} 
As usual, one seeks to estimate the parameters $\sigma^2$ and $\{(\omega_k, \bmu_k)\}_{k=1}^K$ from $N$ i.i.d.~samples $\x_n$. The uncentered covariance matrix reads

\begin{equation}
  \label{cov-GMM}
\mSigma := \E\left[\x_n\x_n^\top\right] = \sum_{k=1}^K \omega_k\,\bmu_k\bmu_k^\top  + \sigma^2\mI_P = \mM_2 + \sigma^2\mI_P ,
\end{equation}
where $\mM_2=\sum_{k=1}^K \omega_k\bmu_k\bmu_k^\top$ encodes the ``inter-cluster'' covariance.

Since $\mM_2$ carries information on the mixture structure,
it makes sense to whiten the data with respect to this covariance, which amounts to constructing from $\mSigma$ a $K\times P$ \textit{whitening matrix} $\mW$ such that $\mW \mM_2 \mW^\top=\mI_K$.
By construction, such a linear map performs two things at once: (i) dimensionality reduction, since it maps data onto the $K$-dimensional ``signal'' subspace; (ii) orthogonalization of the whitened means $\W\bmu_k$ (see Section \ref{sec:whiteningLD} for more details).

Now, Hsu and Kakade showed in \cite{hsu2013learning} that one can exploit third-order moments of $\x_n$ to build an order-3 tensor $\cM_3$ having a structure analogous to $\mM_2$, that is, $\cM_3 \;=\; \sum_{k=1}^K \omega_k\,\bmu_k^{\otimes 3},$ where $\bmu_k^{\otimes 3} := \bmu_k \otimes \bmu_k \otimes \bmu_k.$
This fact alone is relevant for estimation, since this decomposition of $\cM_3$ into rank-one tensors is notoriously unique under mild constraints, unlike that of $\mM_2$.
Yet, computing the above decomposition is generally challenging, and this is where  whitening helps: after whitening the data, one obtains a corresponding (transformed) tensor which is \textit{orthogonally decomposable}, namely 
\begin{equation}
\label{dec-tens}
\cM_3(\mW,\mW,\mW) :=
\sum_{k=1}^K \omega_k\,(\mW \, \bmu_k)^{\otimes 3}
= \sum_{k=1}^K \omega_k^{-1/2}\,\v_k^{\otimes 3},
\end{equation}
with $\{\v_k:=\omega_k^{1/2}\mW\bmu_k\}_{k=1}^K$  orthonormal (as proven in \Cref{sec:whiteningLD}). 
Hence, the estimation problem reduces to an orthogonal tensor decomposition on $\RR^K$, and thus is solvable by stable spectral algorithms; see \cite{anandkumar2014tensor}, \cite{hsu2013learning}. This whitening-plus-tensor-decomposition procedure can serve as a stand-alone estimator or as an initialization to a maximum likelihood estimation algorithm such as EM.

However, in practice the whitening matrix is built from the sample covariance matrix $\widehat\mSigma:= \frac{1}{N} \x_n \x_n^\top$, which is not always an accurate estimate of $\mSigma$ \cite{leDOITwolf2004}.
This happens in particular in many modern applications where data are high-dimensional and (relatively) scarce, a situation which is formalized by the so-called \textit{large-dimensional regime} (LDR), where both $P$ (data dimension) and $N$ (sample size) tend to infinity, with fixed $P/N= c\in(0,\infty)$. In this (common) regime, random matrix theory (RMT) models $\widehat{\mSigma}$ as a low-rank ``signal’’ perturbed by isotropic noise (a spiked model) and predicts that it is \textit{spectrally distorted}: its eigenvalues follow a Marchenko-Pastur law and its leading eigenvectors are rotated away from the true ones \cite{bai2010spectral}, \cite{couillet2022random}, \cite{johnstone2001spiked}, \cite{benaych2011eigenvalues}, \cite{paul2007asymptotics}. 
As a consequence, the whitening matrix $\mW$ built from $\widehat{\mSigma}$
fails to orthogonalize the GMM component means $\bmu_k$, and the whitened tensor moment $\cM_3$ loses its orthogonal structure, which in turn degrades tensor-based GMM estimators.

Our goal in this paper is twofold: (i) precisely quantifying the impact of these distortions on 
the whitening step employed in these estimators, when operating in the LDR; (ii) proposing a corrected whitening based on RMT results for mitigating this impact. Specifically, we show that 
standard whitening no longer orthogonalizes the GMM clusters' means $\bmu_k$ in the LDR, by giving exact closed-form limits of the dot products between whitened means. We then construct an RMT-corrected whitening operator that reweighs empirical eigenpairs to 
restore orthogonality of the transformed means. Finally, as an illustration, we plug the corrected whitening into the \textsc{LearnGMM} algorithm of \cite{hsu2013learning}, one of the simplest of tensor-based spherical GMM estimators.
Our numerical results clearly show the degradation of its performance in the LDR, and the improvement brought by our corrected whitening.

 \section{Model and assumptions}\label{sec:model}

We will henceforth consider the spherical GMM specified in \eqref{GMM}.

\paragraph{Assumptions.}
We work under the following regime:
\begin{itemize}

\item[(A1)] \textbf{Large-dimensional regime (LDR).} The sample number $N$ and the dimension $P$ grow large at the same rate with fixed $P/N=c\in(0,\infty)$.
  
\end{itemize}
To avoid trivial cases where estimation becomes either too easy or impossible in this regime, we also impose the following assumptions: 

\begin{itemize}
\item[(A2)] \textbf{Non-degeneracy of signal subspace.} We assume that $\rank(\mM_2)=K$ for $\mM_2$ the low rank term in \eqref{cov-GMM}.

\item[(A3)] \textbf{Non-vanishing, non-degenerate components.} The variance $\sigma^2$, and the prior probabilities $\{\omega_k\}_{k=1}^K$ are all nonzero and constant with respect to $P$ and $N$, for fixed $K$.

\item[(A4)] \textbf{Constant signal strength.} The norms of the GMM means $\bmu_k$ and their alignments $\langle \bmu_i, \bmu_j \rangle$ are assumed to be constant with respect to $N$ and $P$.
\end{itemize}

\paragraph{Phase transition of sample covariance in the LDR.} Consider the following spectral decompositions for the covariance $\mSigma$ given in \eqref{cov-GMM}  and its empirical version $\widehat\mSigma=\frac{1}{N} \x_n \x_n^\top$:
\begin{equation}
\label{eq:spectral-decomp-Sigma-hatSigma}
    \mSigma = \sum_{k=1}^P \lambda_k\u_k\u_k^\T, \quad \widehat\mSigma = \sum_{k=1}^P \hat\lambda_k\hat\u_k\hat\u_k^\T,
\end{equation}
where $\{(\lambda_k,\u_k)\}_{k=1}^{P}, \{(\hat\lambda_k,\hat\u_k)\}_{k=1}^{P}$ are respectively the eigenpairs of $\mSigma,\widehat\mSigma$ with  $\lambda_1 \ge \dots \ge \lambda_P,\hat\lambda_1 \ge \dots \ge \hat\lambda_P$. Using~\eqref{cov-GMM}, we get
\begin{equation}
\label{eq:l_k}
    \mSigma = \sum\nolimits_{k=1}^K \sigma^2\ell_k\u_k\u_k^\T+\sigma^2\mI_P,\text{ with } \ell_k={ (\lambda_k - \sigma^2)}/\sigma^2.
\end{equation}
Under (A4), well-known results from RMT indicate a phase transition occurs at $\sqrt{c}$ for $\ell_k$: spikes with $\ell_k>\sqrt{c}$ generate outlier eigenvalues $\hat\lambda_k$ and nonzero asymptotic eigenvector alignments between $\u_k$ and $\hat\u_k$; spikes with $\ell_k\le\sqrt{c}$ are absorbed by the Marchenko–Pastur distribution bulk and  the alignments $\langle\u_k,\hat\u_k\rangle$ vanish \cite{baik2006eigenvalues}, \cite{couillet2022random}.

\begin{theorem}[Spiked–covariance limits in the LDR; adapted from \cite{baik2006eigenvalues}, \cite{couillet2022random}]
\label{thm:spiked_model}
Consider the Gaussian mixture model \eqref{GMM}, and the uncentered covariance $\mSigma$ defined in \eqref{cov-GMM} with its eigenpairs $\{(\lambda_k,\u_k)\}_{k=1}^{P}$ and those $\{(\hat\lambda_k,\hat\u_k)\}_{k=1}^{P}$ of its empirical version $\widehat\mSigma$ as specified in \eqref{eq:spectral-decomp-Sigma-hatSigma}, and  $\ell_k$ as introduced in \eqref{eq:l_k}.
Then, under (A1)--(A4),
\begin{align}
  & \hat\lambda_k \;\toas\; \tilde\lambda_k := 
  \mathbbm{1}_{\{\ell_k>\sqrt{c}\}} \beta_k \lambda_k
  + \mathbbm{1}_{\{\ell_k\le\sqrt{c}\}}\, (1+\sqrt{c})^2, \label{eq:thm-eigs}\\
  & \a^\T\hat\u_k\hat\u_k^\T\b \;\toas\mathbbm{1}_{\{\ell_k>\sqrt{c}\}}\,\psi_k\a^\T\u_k\u_k^\T\b,\quad\forall \a, \b \in\RR^P, \label{eq:thm-overlaps}
\end{align}
where $\beta_k=1+\frac{c}{\ell_k}$ and 
    $\psi_k=1-\frac{\beta_k-1}{\beta_k}\frac{1+\ell_k}{\ell_k}$.
\end{theorem}

\begin{remark}
Directly applying \eqref{eq:thm-overlaps} to $\a=\b=\u_k$ gives the asymptotic (squared) alignment between $\hat\u_k$ and $\u_k$:
\[
\u_k^\top \hat\u_k \hat\u_k^\top \u_k 
= (\hat\u_k^\top \u_k)^2 
\;\toas\; \mathbbm{1}_{\{\ell_k>\sqrt{c}\}}\,\psi_k =: \zeta_k.
\]
\end{remark}

\begin{remark}
When $c>0$ and $\ell_k>\sqrt{c}$, one still has $\psi_k<1$, so empirical eigenvectors do not perfectly align with their population counterparts. In addition, sample eigenvalues are distorted: $\hat\lambda_k \toas \tilde\lambda_k=\beta_k\lambda_k$ for recoverable spikes and $\hat\lambda_k \toas \sigma^2(1+\sqrt{c})^2$ when $\ell_k\le\sqrt{c}$. Both these effects are precisely why whitening becomes ineffective in the LDR.
\end{remark}

\section{Whitening in large dimension}\label{sec:whiteningLD}

\paragraph{Population whitening matrix.}
According to \eqref{eq:spectral-decomp-Sigma-hatSigma}, the low-rank term $\mM_2$ of $\mSigma$ introduced in \eqref{cov-GMM} admits the spectral decomposition $\mM_2 = \mU_K \,\mGamma \, \mU_K^\top$ where $\U_K := [\u_1,\dots,\u_K] \in \RR^{P \times K}$ and $  \mGamma=\diag(\lambda_1 - \sigma^2,\dots,\lambda_K - \sigma^2)$.
Assumption (A1) implies that matrix $\mGamma \in \RR^{K \times K}$ is non-singular.
The population whitening matrix with respect to the inter-cluster covariance $\mM_2$ is properly defined as 
$\mW \;:=\; \mGamma^{-1/2}\,\U_K^\top\in\RR^{K\times P}$.
Indeed, with this definition,
\begin{equation*}
  \mW \, \mM_2 \, \mW^\top \;=\; \mGamma^{-1/2} \, \U_K^\top(\U_K \, \mGamma \, \U_K^\top)\U_K \, \mGamma^{-1/2} \;=\; \mI_K.
\end{equation*}
Since $\mM_2=\sum_{k=1}^K \omega_k\,\bmu_k\bmu_k^\top$, then $\sum_{k=1}^K \omega_k\,(\mW\bmu_k)(\mW\bmu_k)^\top = \mI_K.$
Using the rescaled and whitened means $\v_k = \omega_k^{1/2}\,\mW\bmu_k \in \RR^K$ introduced in \eqref{dec-tens}, the previous identity rewrites as $\sum_{k=1}^K \v_k\v_k^\top \;=\; \mI_K,$ which means that the family $\{\v_k\}_{k=1}^K$ is orthonormal.  
It follows that $\|\mW\bmu_k\|^2 = \omega_k^{-1}$ and $\langle \mW\bmu_i,\ \mW\bmu_j\rangle = 0$ for $i\neq j$.

\paragraph{Whitening matrix estimation.}
Let $\widehat\mM_2=\widehat\mSigma-\hat\sigma^2\mI_P$ (for some chosen variance estimator $\hat\sigma^2$).
Upon computing the spectral decomposition of $\widehat\mM_2$ and keeping its $K$ dominant eigenpairs, $\widehat\mM_2 \approx\widehat\U_K\,\hat\mGamma\,\widehat\U_K^\top$, one can estimate the whitening matrix as $\widehat\mW := \hat\mGamma^{-1/2}\, \widehat\U_K^\top \in\RR^{K\times P}$,
where the subscript $K$ indicates that only the $K$ leading eigenpairs are retained.
In the classical regime ($P$ constant and $N\to\infty$), $\widehat\mW\toas \mW$ and the orthogonality between whitened means is (asymptotically) verified.
Yet, this property breaks down in the LDR, as seen next.

\subsection{Residual alignment and phase transition}\label{ssec:resalign}

We now examine the loss of orthogonality between whitened means that arises in the LDR, as a consequence of Theorem \ref{thm:spiked_model}. Given the above estimate of the whitening matrix $\widehat\mW$, the main quantity of interest is  the empirical residual alignment between component means defined below:
\begin{equation}
\label{eq:def-rho}
\rho_{i,j}\ :=\ \frac{\big|\langle \widehat{\mW}\bmu_i,\ \widehat{\mW}\bmu_{j}\rangle\big|}
{\|\widehat{\mW}\bmu_i\|\ \|\widehat{\mW}\bmu_{j}\|},
\end{equation}
for $i\neq j\in\{1,\ldots,K\}$.

Since $\widehat\mW$ is built from a reduced spectral decomposition of $\widehat\mM_2 = \widehat\mSigma - \hat\sigma^2 \mI_P$, it depends on the chosen variance estimator $\hat\sigma^2$. 
In the classical regime, simply taking $\hat\sigma^2=\hat\lambda_p$ with any $p>K$ yields a consistent estimator of $\sigma^2$. Yet, this no longer holds in the LDR where $\hat\lambda_p$ with $p>K$ are distributed according to the Marchenko-Pastur law on the support $[\sigma^2(1 - \sqrt{c})^2, \sigma^2(1 + \sqrt{c})^2]$ \cite{bai2010spectral}.
{ Nevertheless, using Assumptions (A1) and (A4) one can show that the estimator $\hat{\sigma}^2 = \frac{1}{NP} \sum_{n=1}^N \|\x_n\|^2$ is consistent, even when the data are not centered.}
Hence, in the following we consider that $\sigma^2$ is consistently estimated, focusing only on the effect of the spectral distortion described by Theorem \ref{thm:spiked_model}.

\begin{lemma}[Limit of residual alignment]
\label{prop:residual-alignment}
Assume (A1)–(A4) and consider the same notation as in Theorem~\ref{thm:spiked_model}. Set $d_k:= \mathbbm{1}_{\ell_k \ge \sqrt{c}} \, \psi_k/(\tilde{\lambda}_k - \sigma^2) = \mathbbm{1}_{\ell_k \ge \sqrt{c}} \, \psi_k/(\beta_k \lambda_k - \sigma^2)$, 
$\mD:=\diag(d_1,\ldots,d_K)$ and $\mT:= \mD^{1/2} \, \U_K^\top$.
Then, we have
\begin{equation}
\label{eq:alignment_whitened_means}
\langle \widehat{\mW}\bmu_i,\ \widehat{\mW}\bmu_{j}\rangle 
   \toas\ \langle \mT \, \bmu_i, \mT \,\bmu_{j} \rangle, 
\end{equation}
and, whenever both $\langle \mT \, \bmu_i, \mT \,\bmu_{i} \rangle > 0$ and $\langle \mT \, \bmu_j, \mT \,\bmu_{j} \rangle > 0$,
\begin{equation}
\label{eq:asymptotic_rho}
    \rho_{i,j}\ \toas\ 
\rho^\infty_{i,j} := 
\frac{\big| \langle \mT \, \bmu_i, \mT \,\bmu_{j} \rangle  \big|}
{\sqrt{\langle \mT \, \bmu_i, \mT \,\bmu_{i} \rangle}\ \sqrt{\langle \mT \, \bmu_j, \mT \,\bmu_{j} \rangle}}\,.
\end{equation}

\end{lemma}

\begin{proof}
Notice first that
\begin{align*}
\langle \widehat{\mW} \bmu_i,\ \widehat{\mW} \bmu_j \rangle= \bmu_i^\top \widehat\U_K \widehat\mGamma^{-1} \widehat\U_K^\top \bmu_j = \sum_{k=1}^K \frac{(\bmu_i^\top \hat\u_k)(\hat\u_k^\top \bmu_j)}{\hat\lambda_k-\hat\sigma^2}.
\end{align*}
The asymptotic limit of $(\bmu_i^\top \hat\u_k)(\hat\u_k^\top \bmu_j)$ is given by letting $\a=\bmu_i,\b=\bmu_j$ in \eqref{eq:thm-overlaps}. Retrieving also the limiting value of $\hat\lambda_k$ from \eqref{eq:thm-eigs} and recalling that $\hat\sigma^2\toas\sigma^2$, we get
\begin{align*}
\langle \widehat{\mW} \bmu_i,\ \widehat{\mW} \bmu_j \rangle \toas& \sum_{k=1}^K \frac{\mathbbm{1}_{\ell_k \ge \sqrt{c}} \, \psi_k(\bmu_i^\top \u_k)(\u_k^\top \bmu_j)}{\tilde\lambda_k^2-\sigma^2}\\
=&\sum_{k=1}^K d_k(\bmu_i^\top \u_k)(\u_k^\top \bmu_j)=\langle \mT \, \bmu_i, \mT \,\bmu_{j} \rangle,
\end{align*}
thus recovering \eqref{eq:alignment_whitened_means}.
\end{proof}

\begin{remark}[Classical regime]
Letting $c\to 0$, we get $\psi_k\to 1$ and $\tilde{\lambda}_k\to\lambda_k$. Hence $\mT \to \mW$ in the classical regime, implying $\rho_{i,j} \to  0$ for $i\neq j$.
\end{remark}

\begin{example}
Consider the case $K = 2$ with $\omega_1=\omega_2=1/2$, and suppose that $\bmu_1,\bmu_2$ have unit norm and alignment $\langle\bmu_1,\bmu_2\rangle=\cos\theta\in [0,1)$. With some algebraic calculation, we get $\lambda_1=\cos^2(\theta/2)+\sigma^2, \lambda_2=\sin^2(\theta/2)+\sigma^2$, from which it follows that $\ell_1=\cos^2(\theta/2)/\sigma^2, \ell_2=\sin^2(\theta/2)/\sigma^2$.
When $\ell_1,\ell_2>\sqrt{c}$, applying \Cref{prop:residual-alignment} leads to
\begin{equation*}
    \rho_{1,2}\toas \frac{q_1-q_2}{q_1+q_2}, \quad \text{where} \quad q_{k}= \frac{\ell_k^4-c \ell_k^2}{\ell_k(\ell_k+c)(\ell_k^2+c\ell_k+c)}.
\end{equation*}
It can be seen from the above equation that the alignment $\rho_{1,2}$ between the whitened means $\widehat{\mW}\bmu_1, \widehat{\mW}\bmu_1$ is asymptotically non-zero in the LDR with $c>0$, unless $\ell_1=\ell_2$, in which case the original means $\bmu_1,\bmu_2$ are already orthogonal, as $\ell_1=\ell_2$ implies $\theta/2=\pi/4$, meaning $\theta=\pi/2$.

A numerical example is shown in Figure \ref{fig:residual-alignment}.
We generate $N=5000$ observations $\x_n$ following \eqref{GMM} with $K=2$ and $P=500$. Here, $\omega_1 = 0.567 $ and $\omega_2 = 0.433$ and the means $\bmu_1,\bmu_2$ are unit-norm vectors satisfying $\bmu_1^\top \bmu_2 = 1/4$.

We compute the estimated squared alignments $\hat\zeta_k = (\hat\u_k^\top \u_k)^2$ and compare them with the theoretical limiting quantities $\zeta_k$.
The estimated residual alignment $\rho_{1,2}$ between $\widehat{\mW}\bmu_1, \widehat{\mW}\bmu_2$ is computed by \eqref{eq:def-rho} and compared to its theoretical counterpart $\rho^\infty_{1,2}$ given in \eqref{eq:asymptotic_rho}.
This is repeated for different levels of signal-to-noise ratio (SNR), here defined as $\text{SNR} := 1/\sigma^2$.
The estimated quantities are averaged over of $25$ independent Monte Carlo runs (of $N=5000$ samples each).
Gray dashed vertical lines mark the two critical SNRs $S_k^\star := (\sqrt{c} + 1)/\lambda_k$ dictated by the spiked model.

    For $1/\sigma^2 <S_1^\star$, both spikes are subcritical and no signal can be found ($\zeta_1=\zeta_2=0$). 
    For $S_1^\star< 1/\sigma^2 <S_2^\star$, only the first eigenvector $\u_1$ is partially recovered ($\zeta_1>0 = \zeta_2$) and the whitened means become collinear. 
    For $1/\sigma^2 >S_2^\star$, both $\u_1, \u_2$ are partially recovered ($\zeta_1,\zeta_2>0$) but the whitened means still have nonzero residual alignment as predicted by Lemma~\ref{prop:residual-alignment}.
    The observations closely track the theoretical curves, apart from some oscillation around the second phase transition, as usual in finite dimension.

\end{example}

\begin{figure}[!t]
    \centering
    \input{tikz/res_align.tikz}
    \caption{Residual and eigenvector alignments in the LDR ($K=2$).}
    \label{fig:residual-alignment}
\end{figure}
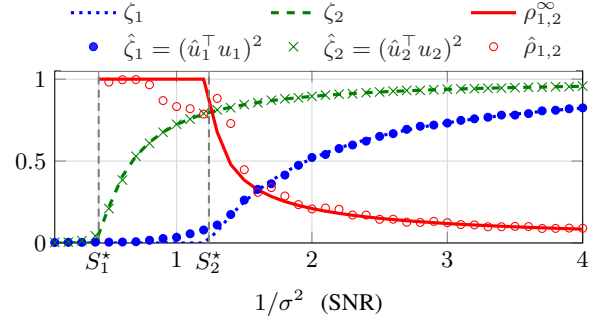
\vspace{-4mm}

\subsection{Corrected whitening}\label{ssec:def-cwhiten}

Lemma~\ref{prop:residual-alignment} shows that the empirical whitening map \(\hat{\mW}\) fails to orthogonalize the means in the LDR because both eigenvalues and eigenvector alignments are distorted according to \eqref{eq:thm-eigs}--\eqref{eq:thm-overlaps}, {implying that the coordinates of the whitened means are  biased by some scalar factors.} However, these factors can be predicted from \eqref{eq:thm-eigs}--\eqref{eq:thm-overlaps}. Hence, one can perform the following ``scalar'' correction: replace each empirical eigenvalue \(\hat\lambda_k\) by a consistent estimator of $\lambda_k$ and compensate for eigenvector alignment $\hat\u_k^\top \u_k$ shrinkage by dividing it by a consistence estimator of$\psi_k$. Similarly corrected estimators based on RMT results have been derived in other contexts (see, e.~g.,~\cite{nadakuditi2014optshrink}, \cite{gavish2014optimal}).

{Concretely, we proceed as follows. We may express $\lambda_k$ and $\beta_k$ as functions of $\ell_k$ in \eqref{eq:thm-eigs}, and then express $\ell_k$ as a function of $\tilde{\lambda}_k$ (when $\ell_k > \sqrt{c}$), which is tantamount to choosing the positive root of $\ell_k^2 + (1+c - \tilde\lambda_k/\sigma^2 )\ell_k + c = 0$. By replacing $\sigma$ by $\hat{\sigma}$ and $\tilde\lambda_k$ by $\hat{\lambda}_k$, we obtain the expression of the corrected spike:}
\begin{equation*}
\hat\ell_k^{\,\rm c}
\;=\; \frac{1}{2} \left(
\displaystyle \frac{\hat\lambda_k}{\hat\sigma^2} - (1+c)
\;+\;
\sqrt{\Big(\frac{\hat\lambda_k}{\hat\sigma^2} - (1+c)\Big)^2 - 4c}
\right).
\end{equation*}
Then the corrected estimator of $\lambda_k$ is 
$\hat\lambda_k^{\,\rm c} \;=\; \hat\sigma^2\,(1+\hat\ell_k^{\,\rm c})$.
Next, define the finite-$(P,N)$ corrected versions of the overlap factors:
\begin{equation*}
\hat\beta_k^{\,\rm c} \;=\; 1 + \frac{c}{\hat\ell_k^{\,\rm c}},
\qquad
\hat\psi_k^{\,\rm c}
\;=\;
1 - \frac{\hat\beta_k^{\,\rm c}-1}{\hat\beta_k^{\,\rm c}}
\,\frac{1+\hat\ell_k^{\,\rm c}}{\hat\ell_k^{\,\rm c}}.
\end{equation*}
Using these quantities, we can correct both eigenvalue inflation and eigenvector shrinkage via the diagonal scaling
\begin{equation}
\label{eq:scalar-correc}
\mPhi \; := \; \diag\!\Bigg(
\frac{\mathbbm{1}_{\{\hat\ell_1^{\,\rm c}>\sqrt{c}\}}}{\sqrt{(\hat\lambda_1^{\,\rm c}-\hat\sigma^2)\,\hat\psi_1^{\,\rm c}}},
\;\ldots,\;
\frac{\mathbbm{1}_{\{\hat\ell_K^{\,\rm c}>\sqrt{c}\}}}{\sqrt{(\hat\lambda_K^{\,\rm c}-\hat\sigma^2)\,\hat\psi_K^{\,\rm c}}}
\Bigg),
\end{equation}
by simply defining the \emph{corrected whitening} matrix as $\widehat{\mW}^{\,\rm c} := \mPhi \widehat\U_K^\top \, \in \, \mathbb{R}^{K\times P}$. These definitions directly lead to a proof of the following theorem, which is our main result.

\begin{theorem}[Asymptotic dot products under corrected whitening]
Assume (A1)–(A4) and consider the same notation as in Theorem~\ref{thm:spiked_model}.
Let
$\mH := \diag(\mathbbm{1}_{\{\ell_1>\sqrt{c}\}}, \ldots, \mathbbm{1}_{\{\ell_K>\sqrt{c}\}})$.
Then, for any $i,j\in\{1,\ldots,K\}$,
\begin{equation*}
\big\langle \widehat{\mW}^{\,\rm c}\bmu_i,\ \widehat{\mW}^{\,\rm c}\bmu_j\big\rangle
\ \xrightarrow{\mathrm{a.s.}}\
\bmu_i^\top \, \U_K\,\H\, \mGamma^{-1} \U_K^\top \, \bmu_j.
\end{equation*}
In particular, if $\ell_k>\sqrt{c}$ for all $k=1,\ldots,K$, then
\begin{equation*}
\big\langle \widehat{\mW}^{\,\rm c}\bmu_i,\ \widehat{\mW}^{\,\rm c}\bmu_j\big\rangle
\ \xrightarrow{\mathrm{a.s.}}\ 
\big\langle \mW \bmu_i,\ \mW \bmu_j\big\rangle = \omega_i^{-1}\,\delta_{ij}.
\end{equation*}
\end{theorem}

\section{Application to spherical GMM}\label{sec:application}

We now illustrate the impact of our results, and in particular of our corrected whitening matrix, on the performance of a tensor-based algorithm for spherical GMM estimation. 
For concreteness, we will take the \textsc{LearnGMM} algorithm of \cite{hsu2013learning}, which is arguably one of the simplest algorithms of this kind.
In \textsc{LearnGMM}, after the estimated whitening matrix $\widehat{\mW}$ is obtained from $\{\x_n\}_{n=1}^N$, a new set $\{\x_n\}_{n=N+1}^{2N}$ of $N$ independent samples is required to estimate the whitened tensor moment $\cM_3(\mW,\mW,\mW)$ defined in \eqref{dec-tens}.  Letting $\bxi_n=\widehat{\W}\x_n$ for $n\in\{N+1,\ldots,2N\}$, the estimator $\widehat{\cM}_3(\widehat\mW,\widehat\mW,\widehat\mW)$  of  $\cM_3(\mW,\mW,\mW)$ in \textsc{LearnGMM} is given by 
\begin{multline}
\label{est-M3}
\widehat{\cM}_3(\widehat{\mW},\widehat{\mW},\widehat{\mW})
:= \frac{1}{N}\sum\nolimits_{n=N+1}^{2N} \bxi_n^{\otimes 3} \\
- \hat\sigma^2 \sum\nolimits_{k=1}^K 
\Bigl(
   \bar\bxi\otimes\e_k\otimes\e_k
   + \e_k\otimes\bar\bxi\otimes\e_k
   + \e_k\otimes\e_k\otimes\bar\bxi
  \Bigr),
\end{multline}
where $\bar\bxi=\frac{1}{N}\sum_{n=N+1}^{2N}\bxi_n$ and $\e_k$ denotes the $k$th canonical basis vector of $\RR^K$. The authors of \cite{hsu2013learning} then argued that $\widehat{\cM}_3(\widehat{\mW},\widehat{\mW},\widehat{\mW})\to \cM_3(\mW,\mW,\mW)$ as $N \to \infty$ with fixed $P$.
One can then apply a spectral decomposition algorithm to $\cM_3(\mW,\mW,\mW)$, thereby recovering the vectors $\v_k$, which can be transformed back into estimates of the component means $\bmu_k$.
In particular, the decomposition procedure used in \textsc{LearnGMM} consists in computing the spectral decomposition of a symmetric matrix obtained by contracting  $\cM_3(\mW,\mW,\mW)$ with a random vector.

\subsection{Corrected whitened third moment} \label{sec:tensordec}

In the LDR, the estimator $\widehat{\cM}_3(\widehat{\mW},\widehat{\mW},\widehat{\mW})$ of \eqref{est-M3} no longer converges to the whitened moment tensor $\cM_3(\mW,\mW,\mW)$, as a consequence of \Cref{prop:residual-alignment}. 
Nonetheless, we claim in the following that this property can be recovered by simply replacing the whitening matrix by our corrected version $\widehat\mW^{\,\rm c}$.
Specifically, define the corrected estimator $ \widehat{\cM}^{\,\rm c}_3(\widehat{\mW}^{\,\rm c},\widehat{\mW}^{\,\rm c},\widehat{\mW}^{\,\rm c})$ by modifying the formula \eqref{est-M3} as follows: 
\begin{multline}
\label{est-M3-corr}
\widehat{\cM}^{\,\rm c}_3(\widehat{\mW}^{\,\rm c},\widehat{\mW}^{\,\rm c},\widehat{\mW}^{\,\rm c})
:= \frac{1}{N}\sum\nolimits_{n=N+1}^{2N} (\bxi_n^{\,\rm c})^{\otimes 3} \\
- 
\sum\nolimits_{k=1}^K \frac{\hat\sigma^2 }{\hat\lambda_k^{\rm c}\hat\psi_k^{\rm c}}  \Bigl(
   \bar\bxi^{\,\rm c}\otimes\e_k\otimes\e_k
   + \e_k\otimes\bar\bxi^{\,\rm c}\otimes\e_k
   + \e_k\otimes\e_k\otimes\bar\bxi^{\,\rm c}
  \Bigr),
\end{multline}
where $\bxi_n^{\,\rm c} = \widehat{\mW}^{\,\rm c}\x_n$ and $\bar\bxi^{\,\rm c}=\frac{1}{N}\sum_{n=N+1}^{2N}\bxi_n^{\,\rm c}$.

\begin{theorem}[Consistency of the corrected whitened third moment]
Assume (A1)–(A4) and consider the same notation as in Theorem~\ref{thm:spiked_model}. If \(\ell_k>\sqrt{c}\) for all \(k\in\{1,\ldots,K\}\), then
\[
 \widehat{\cM}^{\,\rm c}_3(\widehat{\mW}^{\,\rm c},\widehat{\mW}^{\,\rm c},\widehat{\mW}^{\,\rm c})
\xrightarrow{\mathrm{a.s.}}\ \cM_3(\mW,\mW,\mW).
\]
\end{theorem}
\begin{remark}[Estimation Pseudocode]
\label{rem:pseudocode}
The above consistency result justifies the estimation procedure obtained by combining our corrected estimates with the tensor decomposition and GMM parameter estimation steps of \textsc{LearnGMM} \cite{hsu2013learning}. 
We summarize in Algorithm 1 below the resulting algorithmic procedure in pseudocode form, using the same sample-splitting strategy as in \textsc{LearnGMM}.
\end{remark}

\begin{tcolorbox}[enhanced,breakable,sharp corners,boxrule=0.4pt,
  colback=white,colframe=black,title=\footnotesize{Algorithm 1: \textsc{LearnGMM} corrected for large-dimensional regime}]
\footnotesize
\textbf{Input:} $2N$ samples $\{\x_n\}_{n=1}^{2N}$, number of components $K$, ratio $c=P/N$.\\
\textbf{Output:} Estimates of $(\omega_k,\bmu_k)_{k=1}^K$ (up to permutation) as in \textsc{LearnGMM}.

\vspace{-1mm}
\begin{enumerate}[1)]
\item \textbf{Estimate variance and covariance (first half-sample).}
\[
\hat\sigma^2 \;=\; \frac{1}{NP}\sum_{n=1}^{N}\|\x_n\|^2,
\qquad
\widehat\mSigma \;=\; \frac{1}{N}\sum_{n=1}^{N}\x_n\x_n^\top.
\]

\item \textbf{Top-$K$ eigendecomposition.} Compute the $K$ leading eigenpairs $(\hat\lambda_k,\hat\u_k)_{k=1}^K$ of $\widehat\mSigma$, and set $\widehat\U_K=[\hat\u_1,\ldots,\hat\u_K]$.

\item \textbf{Corrected whitening matrix.} Compute $\mPhi$ as in \eqref{eq:scalar-correc} and then form
\[
\widehat{\mW}^{\,\rm c}=\mPhi\,\widehat\U_K^\top \in \RR^{K\times P}.
\]

\item \textbf{Whiten independent samples (second half-sample).} For $n=N+1,\ldots,2N$,
\[
\bxi_n^{\,\rm c}=\widehat{\mW}^{\,\rm c}\x_n,
\qquad
\bar\bxi^{\,\rm c}=\frac{1}{N}\sum_{n=N+1}^{2N}\bxi_n^{\,\rm c}.
\]

\item \textbf{Corrected whitened third moment tensor.} Compute $\widehat{\cM}_3^{\,\rm c}$ as in \eqref{est-M3-corr}.

\item \textbf{Spectral step (as in \textsc{LearnGMM}).} Draw a random $\theta\in\RR^K$ uniformly from the unit sphere and contract it with $\widehat{\cM}_3^{\,\rm c}$:
\[
[\mA(\theta)]_{ij}=\sum_{k=1}^K \big[\widehat{\cM}_3^{\,\rm c}\big]_{ijk}\,\theta_k.
\]
Compute the eigenpairs of $\mA(\theta)$ to recover $\v_k$ (up to sign/permutation).

\item \textbf{Parameter estimates.} For each recovered component $\hat\v_k$, compute 
\[
\hat\omega_k
\;=\;
\big(\theta^\top \hat\v_k\big)^2,
\qquad
\hat\bmu_k
\;=\;
\frac{1}{\sqrt{\hat\omega_k}}\,
(\widehat{\mW}^{\,\rm c})^{\dagger}\hat\v_k,
\]
where $(\cdot)^\dagger$ denotes the Moore-Penrose pseudoinverse.
\end{enumerate}
\end{tcolorbox}
\vspace{-1mm}

\begin{remark}
\label{rem:thetamulti}
Note that in practice one can draw several independent vectors $\theta$ in Step~1 of Algorithm 1 and then keep the best-conditioned matrix $\mA(\theta)$, or the one with the most well-separated eigenvalues, as done in the original \textsc{LearnGMM} procedure \cite{hsu2013learning}.
\end{remark}

\begin{remark}
\label{rem:consist}
 While consistency of the low-dimensional object $\cM_3(\mW,\mW,\mW)$ is ensured by using $\widehat{\mW}^{\,\rm c}$, the component means $\bmu_k$ are \emph{not} consistently estimated because the signal subspace $\text{span}\{\u_1,\ldots,\u_K\}$ cannot be perfectly recovered in the LDR, as per Theorem \ref{thm:spiked_model}.
 Still, it can be shown that our correction allows recovering their projection onto the estimated subspace $\widehat\mU_K \widehat\mU_K^\top \bmu_k$, which is the best that can be achieved when whitening is used.
\end{remark}

\subsection{Numerical illustration}\label{ssec:pipeline}

{We consider the estimation of the means in a spherical GMM with $K=2$, $N=2500$, equal weights $\omega_1 = \omega_2 = 0.5$ and random unit-norm means satisfying $\bmu_1^\top \bmu_2 = 0.5$. We compare the performance of the estimators described in Section~\ref{sec:tensordec}, with and without corrected whitening, by computing $\|\hat\bmu_k - \bmu_k \|^2$ for $k=1,2$. Fig.~\ref{fig:alg-perf-snr} displays the averaged results over 60 Monte Carlo realizations of the dataset, for $P=2500$ (left) and $P=500$ (right). The dashed vertical lines indicate the critical SNR levels defined in Section~\ref{ssec:resalign}.
As predicted, both methods perform poorly before the second phase transition $S^\star_2$, since the eigenvectors $\u_1,\u_2$ cannot be both weakly recovered in this regime. Above $S_2^\star$, both spikes are recoverable, yet the original algorithm using standard whitening still incurs in substantial errors due to a high residual alignment between whitened means. By contrast, the corrected whitening matrix restores (asymptotic) orthogonality, thereby significantly reducing estimation errors. At very high SNR, the residual alignment decays and both methods perform nearly identically. As seen in Fig.~\ref{fig:alg-perf-snr} (right), significant improvement with the corrected whitening is obtained on a smaller interval for $P=500$, since $c$ then gets closer to the classical regime $c\to 0$.}

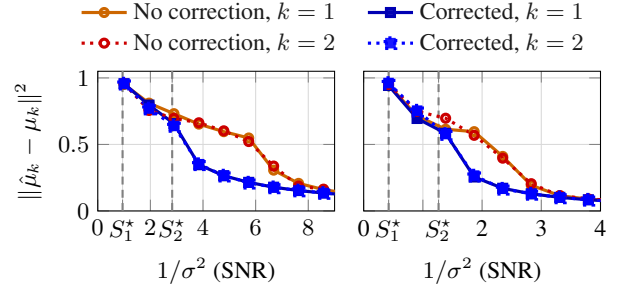
\begin{figure}[!t]
\centering
\input{tikz/sim.tikz}
\caption{GMM means estimation with and without corrected whitening.}
\label{fig:alg-perf-snr}
\end{figure}

\section{Conclusion}

By exploiting results pertaining to spiked covariance matrix models coming from random matrix theory, we precisely characterized the loss of orthogonality of the whitened GMM means that takes place in the large-dimensional regime.
This result allowed us to propose a simple scalar correction that restores their asymptotic orthogonality, leading to a significant performance improvement, especially under low signal-to-noise ratio, as illustrated by our numerical experiments.

\vfill\pagebreak

\bibliographystyle{IEEEbib}
\bibliography{strings, refs}

\end{document}

%% file: math_commands_I.tex
\newcommand{\RR}{\mathbb{R}}
\newcommand{\T}{ {\sf T} }
\newcommand{\E}{\mathbb{E}}

\newcommand{\x}{{\mathbf{x}}}

\newcommand{\e}{{\mathbf{e}}}

\newcommand{\U}{{\mathbf{U}}}

\newcommand{\W}{{\mathbf{W}}}

\newcommand{\bmu}{ \boldsymbol{\mu} }

\newcommand{\bxi}{ \boldsymbol{\xi} }

\DeclareMathOperator{\rank}{rank}

\newcommand{\cM}{{\boldsymbol{\mathcal{M}}}}

\renewcommand{\b}{{\mathbf{b}}}
\renewcommand{\u}{{\mathbf{u}}}

\renewcommand{\v}{{\mathbf{v}}}

\renewcommand{\a}{{\mathbf{a}}}

\renewcommand{\H}{{\mathbf{H}}}
\newcommand{\toas}{\overset{a.s.}{\longrightarrow}}

\DeclareMathOperator{\diag}{{\rm Diag}}

\definecolor{RED}{rgb}{0.7,0,0}
\definecolor{BLUE}{rgb}{0,0,0.69}
\definecolor{GREEN}{rgb}{0,0.6,0}
\definecolor{PURPLE}{rgb}{0.69,0,0.8}

\theoremstyle{theorem}
\newtheorem{theorem}{Theorem}

\newtheorem{lemma}{Lemma}

\theoremstyle{remark}
\newtheorem{remark}{Remark}

\theoremstyle{definition}

\newtheorem{example}{Example}

\newcommand{\mat}[1]{\mathbf{#1}}

\newcommand{\mA}{\mat{A}}

\newcommand{\mD}{\mat{D}}

\newcommand{\mH}{\mat{H}}
\newcommand{\mI}{\mat{I}}

\newcommand{\mM}{\mat{M}}

\newcommand{\mT}{\mat{T}}
\newcommand{\mU}{\mat{U}}

\newcommand{\mW}{\mat{W}}

\newcommand{\mSigma}{\mat{\Sigma}}

\newcommand{\mGamma}{\mat{\Gamma}}
\newcommand{\mPhi}{\mat{\Phi}}

%% file: tikz/res_align.tikz
\begin{tikzpicture}
\begin{axis}[
  width=\linewidth, height=0.45\linewidth,
  xmin=0.1, xmax=4.0, ymin=0.0, ymax=1.05,
  xlabel={$1/\sigma^2$ \, (SNR)},
  grid=major, grid style={draw=gray!30},
  legend cell align=left,
  legend columns=3,
  legend style={
  font=\small,
  draw=none,
  fill=white,
  fill opacity=0.85,
  at={(0.5,1.0)}, anchor=south,
  column sep=0.05ex  
  },
  extra x ticks={0.4245811816762847, 1.239133605130375},
  extra x tick labels={$S^\star_1$, $S^\star_2$},
]

\addplot+[no marks, very thick, blue, dotted] coordinates { (0.1 , 0) (0.2 , 0) (0.3 , 0) (0.4 , 0) (0.5 , 0) (0.6 , 0) (0.7 , 0) (0.8 , 0) (0.9 , 0) (1 , 0) (1.1 , 0) (1.2 , 0) (1.3 , 0.0702681) (1.4 , 0.169238) (1.5 , 0.251799) (1.6 , 0.321482) (1.7 , 0.380904) (1.8 , 0.432042) (1.9 , 0.476413) (2 , 0.515197) (2.1 , 0.549324) (2.2 , 0.579535) (2.3 , 0.606428) (2.4 , 0.630488) (2.5 , 0.652115) (2.6 , 0.671639) (2.7 , 0.689333) (2.8 , 0.70543) (2.9 , 0.720122) (3 , 0.733577) (3.1 , 0.745935) (3.2 , 0.757318) (3.3 , 0.76783) (3.4 , 0.777562) (3.5 , 0.786593) (3.6 , 0.794991) (3.7 , 0.802819) (3.8 , 0.810128) (3.9 , 0.816966) (4 , 0.823375) };
\addlegendentry{$\zeta_1$}

\addplot+[no marks, very thick, green!50!black, dashed] coordinates { (0.1 , 0) (0.2 , 0) (0.3 , 0) (0.4 , 0) (0.5 , 0.219879) (0.6 , 0.407961) (0.7 , 0.530375) (0.8 , 0.615097) (0.9 , 0.67652) (1 , 0.722698) (1.1 , 0.758443) (1.2 , 0.786782) (1.3 , 0.809705) (1.4 , 0.828564) (1.5 , 0.844307) (1.6 , 0.857615) (1.7 , 0.868991) (1.8 , 0.87881) (1.9 , 0.887358) (2 , 0.894859) (2.1 , 0.901486) (2.2 , 0.907378) (2.3 , 0.912646) (2.4 , 0.917382) (2.5 , 0.921659) (2.6 , 0.925538) (2.7 , 0.929071) (2.8 , 0.932301) (2.9 , 0.935264) (3 , 0.937991) (3.1 , 0.940507) (3.2 , 0.942836) (3.3 , 0.944998) (3.4 , 0.947009) (3.5 , 0.948884) (3.6 , 0.950636) (3.7 , 0.952276) (3.8 , 0.953815) (3.9 , 0.955261) (4 , 0.956623) };
\addlegendentry{$\zeta_2$}

\addplot+[no marks, very thick, red] coordinates { (0.4246 , 1) (0.5 , 1) (0.6 , 1) (0.7 , 1) (0.8 , 1) (0.9 , 1) (1 , 1) (1.1 , 1) (1.2 , 1) (1.3 , 0.677013) (1.4 , 0.479333) (1.5 , 0.383692) (1.6 , 0.324634) (1.7 , 0.283629) (1.8 , 0.253101) (1.9 , 0.229288) (2 , 0.21008) (2.1 , 0.194191) (2.2 , 0.180785) (2.3 , 0.169292) (2.4 , 0.15931) (2.5 , 0.150544) (2.6 , 0.142774) (2.7 , 0.135832) (2.8 , 0.129585) (2.9 , 0.123929) (3 , 0.11878) (3.1 , 0.114071) (3.2 , 0.109745) (3.3 , 0.105754) (3.4 , 0.10206) (3.5 , 0.0986302) (3.6 , 0.0954353) (3.7 , 0.0924513) (3.8 , 0.0896573) (3.9 , 0.0870351) (4 , 0.0845687) };
\addlegendentry{$\rho_{1,2}^{\infty}$}

\addplot+[only marks, mark=*, mark size=1.5pt, blue, mark options={fill=blue}] coordinates { (0.1 , 0.00270448) (0.2 , 0.00396203) (0.3 , 0.00361394) (0.4 , 0.00505816) (0.5 , 0.00484515) (0.6 , 0.00568397) (0.7 , 0.00754653) (0.8 , 0.0166326) (0.9 , 0.0271371) (1 , 0.0344202) (1.1 , 0.0560615) (1.2 , 0.0793369) (1.3 , 0.109624) (1.4 , 0.173673) (1.5 , 0.261646) (1.6 , 0.327316) (1.7 , 0.361583) (1.8 , 0.414459) (1.9 , 0.474256) (2 , 0.521864) (2.1 , 0.54015) (2.2 , 0.576411) (2.3 , 0.598064) (2.4 , 0.621698) (2.5 , 0.650598) (2.6 , 0.667393) (2.7 , 0.69392) (2.8 , 0.712599) (2.9 , 0.719773) (3 , 0.73184) (3.1 , 0.748415) (3.2 , 0.758008) (3.3 , 0.767528) (3.4 , 0.780737) (3.5 , 0.780782) (3.6 , 0.789783) (3.7 , 0.80644) (3.8 , 0.810316) (3.9 , 0.819143) (4 , 0.824879) };
\addlegendentry{$\hat\zeta_1 = (\hat u_1^{\top} u_1)^2$}

\addplot+[only marks, mark=x, mark size=2.5pt, draw=green!50!black] coordinates { (0.1 , 0.00335126) (0.2 , 0.00483132) (0.3 , 0.0141124) (0.4 , 0.0437465) (0.5 , 0.210069) (0.6 , 0.384142) (0.7 , 0.52952) (0.8 , 0.608493) (0.9 , 0.676066) (1 , 0.724713) (1.1 , 0.755936) (1.2 , 0.790503) (1.3 , 0.812883) (1.4 , 0.825571) (1.5 , 0.844903) (1.6 , 0.855996) (1.7 , 0.871341) (1.8 , 0.880255) (1.9 , 0.886375) (2 , 0.895159) (2.1 , 0.902909) (2.2 , 0.908409) (2.3 , 0.911098) (2.4 , 0.917496) (2.5 , 0.921728) (2.6 , 0.92676) (2.7 , 0.931086) (2.8 , 0.931327) (2.9 , 0.935569) (3 , 0.938586) (3.1 , 0.940273) (3.2 , 0.941816) (3.3 , 0.944478) (3.4 , 0.945333) (3.5 , 0.948117) (3.6 , 0.950354) (3.7 , 0.951307) (3.8 , 0.953914) (3.9 , 0.954464) (4 , 0.957637) };
\addlegendentry{$\hat\zeta_2 = (\hat u_2^{\top} u_2)^2$}

\addplot+[only marks, mark=o, mark size=1.5pt, red] coordinates { (0.5 , 0.982046) (0.6 , 0.996769) (0.7 , 0.997748) (0.8 , 0.96582) (0.9 , 0.869301) (1 , 0.832721) (1.1 , 0.819688) (1.2 , 0.786572) (1.3 , 0.882533) (1.4 , 0.729233) (1.5 , 0.44745) (1.6 , 0.309179) (1.7 , 0.338308) (1.8 , 0.286565) (1.9 , 0.23412) (2 , 0.207942) (2.1 , 0.213341) (2.2 , 0.205594) (2.3 , 0.170752) (2.4 , 0.17067) (2.5 , 0.144637) (2.6 , 0.145012) (2.7 , 0.126801) (2.8 , 0.125226) (2.9 , 0.131518) (3 , 0.123895) (3.1 , 0.112407) (3.2 , 0.110307) (3.3 , 0.0990433) (3.4 , 0.0990017) (3.5 , 0.101287) (3.6 , 0.100211) (3.7 , 0.0884764) (3.8 , 0.0889184) (3.9 , 0.090797) (4 , 0.0897387) };
\addlegendentry{$\hat\rho_{1,2}$}

\addplot+[gray, densely dashed, thick, forget plot, no marks] coordinates { (0.4245811816762847, 0.0) (0.4245811816762847, 1.05) };

\addplot+[gray, densely dashed, thick, no marks] coordinates { (1.239133605130375, 0.0) (1.239133605130375, 1.05) };

\end{axis}
\end{tikzpicture}

%% file: tikz/sim.tikz
\begin{tikzpicture}
\pgfplotstableread[col sep=comma]{alg_perf_curves500025006005_HALF.csv}\datatable
\pgfplotstableread[col sep=comma]{alg_perf_crit500025006005.txt}\crittable
\pgfplotstablegetelem{0}{snr_crit_1}\of\crittable
\edef\snrcritone{\pgfplotsretval}
\pgfplotstablegetelem{0}{snr_crit_2}\of\crittable
\edef\snrcrittwo{\pgfplotsretval}

\begin{axis}[
  width=0.55\linewidth, height=0.4\linewidth,
  xmin=0, xmax=9, ymin=0, ymax=1.05,
  xlabel={$1/\sigma^2$ (SNR)},
  ylabel={$\|\hat{\mu}_k-\mu_k\|^2$},
  grid=both, grid style={draw=gray!30},
  legend cell align=left,
  legend columns=2,                 
  legend style={
    at={(1,1.05)}, anchor=south,  
    draw=none, font=\small,
    fill=white, fill opacity=0.9
  },
  tick label style={/pgf/number format/fixed},
  every axis plot/.append style={line width=1.2pt},
    extra x ticks={\snrcritone, \snrcrittwo},
  extra x tick labels={$S^\star_1$, $S^\star_2$},
]

\addplot+[gray, densely dashed, thick, forget plot, no marks]
  coordinates {(\snrcritone,0) (\snrcritone,1.05)};

\addplot+[gray, densely dashed, thick, forget plot, no marks]
  coordinates {(\snrcrittwo,0) (\snrcrittwo,1.05)};

\addplot+[color=orange!80!black, mark=o, mark size=1.5pt]
  table[x=snr, y=se_orig_1]{\datatable};
\addlegendentry{No correction, $k=1$ \;\;\;}

\addplot+[color=blue!70!black, mark=square*, mark size=1.5pt, mark options={fill=blue!70!black}]
  table[x=snr, y=se_mod_1]{\datatable};
\addlegendentry{Corrected, $k=1$}

\addplot+[color=red!80!black, dotted, mark=o, mark size=1.5pt, mark options=solid]
  table[x=snr, y=se_orig_2]{\datatable};
\addlegendentry{No correction, $k=2$ \;\;\;}

\addplot+[color=blue, dotted, mark=square*, mark size=2pt]
  table[x=snr, y=se_mod_2]{\datatable};
\addlegendentry{Corrected, $k=2$}


\end{axis}
\end{tikzpicture}
\hskip-3.4cm
\begin{tikzpicture}
\pgfplotstableread[col sep=comma]{alg_perf_curves50005006005NEW_HALF.csv}\datatable
\pgfplotstableread[col sep=comma]{alg_perf_crit50005006005NEW.txt}\crittable
\pgfplotstablegetelem{0}{snr_crit_1}\of\crittable
\edef\snrcritone{\pgfplotsretval}
\pgfplotstablegetelem{0}{snr_crit_2}\of\crittable
\edef\snrcrittwo{\pgfplotsretval}

\begin{axis}[
  width=0.55\linewidth, height=0.4\linewidth,
  xmin=0, xmax=4, ymin=0, ymax=1.05,
  xlabel={$1/\sigma^2$ (SNR)},
  grid=both, grid style={draw=gray!30},
  tick label style={/pgf/number format/fixed},
  every axis plot/.append style={line width=1.2pt},
  xtick={0,1,2,3,4},
  xticklabels={0,,2,3,4},
  extra x ticks={\snrcritone, \snrcrittwo},
  extra x tick labels={$S^\star_1$, $S^\star_2$},
  ytick={0,0.5,1},
  yticklabels={},
]

\addplot+[gray, densely dashed, thick, forget plot, no marks]
  coordinates {(\snrcritone,0) (\snrcritone,1.05)};

\addplot+[gray, densely dashed, thick, forget plot, no marks]
  coordinates {(\snrcrittwo,0) (\snrcrittwo,1.05)};

\addplot+[color=orange!80!black, mark=o, mark size=1.5pt]
  table[x=snr, y=se_orig_1]{\datatable};

\addplot+[color=blue!70!black, mark=square*, mark size=1.5pt, mark options={fill=blue!70!black}]
  table[x=snr, y=se_mod_1]{\datatable};

\addplot+[color=red!80!black, dotted, mark=o, mark size=1.5pt, mark options=solid]
  table[x=snr, y=se_orig_2]{\datatable};

\addplot+[color=blue, dotted, mark=square*, mark size=2pt]
  table[x=snr, y=se_mod_2]{\datatable};


\end{axis}
\end{tikzpicture}